\documentclass[10pt,journal,compsoc]{IEEEtran}
\pdfoutput=1

\usepackage[backend=bibtex, style=ieee]{biblatex}
\DeclareFieldFormat{titlecase}{#1} 
\DeclareFieldFormat{sentencecase}{#1} 

\bibliography{merged}

\usepackage{AVTcommands}

\usepackage{tikz}
\usetikzlibrary{fit,positioning}

\usepackage{amsthm}
\newtheorem{theorem}{Theorem}[section]
\newtheorem{definition}[theorem]{Definition}
\newtheorem{corollary}[theorem]{Corollary}

\newtheorem{algorithm}[theorem]{Algorithm}
\newtheorem{assumption}[theorem]{Assumption}
\newtheorem{lemma}[theorem]{Lemma}

\makeatletter
\newcommand*\bigcdot{\mathpalette\bigcdot@{.5}}
\newcommand*\bigcdot@[2]{\mathbin{\vcenter{\hbox{\scalebox{#2}{$\m@th#1\bullet$}}}}}
\makeatother

\usepackage{enumerate}

\begin{document}

\title{P\'{o}lya Urn Latent Dirichlet Allocation: \\ a doubly sparse massively parallel sampler}

\author{Alexander Terenin, M{\aa}ns Magnusson, Leif Jonsson, David Draper%
\thanks{A. Terenin was with Imperial College London (UK)}%
\thanks{M. Magnusson and L. Jonsson were with Link\"{o}ping University (SE)}%
\thanks{D. Draper was with University of California, Santa Cruz}}

\IEEEpubid{This work has been submitted to the IEEE for possible publication. Copyright may be transferred without notice.}

\IEEEtitleabstractindextext{\begin{abstract}
Latent Dirichlet Allocation (LDA) is a topic model widely used in natural language processing and machine learning.
Most approaches to training the model rely on iterative algorithms, which makes it difficult to run LDA on big corpora that are best analyzed in parallel and distributed computational environments.
Indeed, current approaches to parallel inference either don't converge to the correct posterior or require storage of large dense matrices in memory.
We present a novel sampler that overcomes both problems, and we show that this sampler is faster, both empirically and theoretically, than previous Gibbs samplers for LDA.
We do so by employing a novel P\'{o}lya-urn-based approximation in the sparse partially collapsed sampler for LDA.
We prove that the approximation error vanishes with data size, making our algorithm asymptotically exact, a property of importance for large-scale topic models.
In addition, we show, via an explicit example, that -- contrary to popular belief in the topic modeling literature -- partially collapsed samplers can be more efficient than fully collapsed samplers.
We conclude by comparing the performance of our algorithm with that of other approaches on well-known corpora.
\end{abstract}

}

\maketitle

\section{Introduction}

\IEEEPARstart{L}{atent} Dirichlet Allocation (LDA) \cite{blei03} is a topic model widely used in natural language processing for probabilistically identifying latent semantic themes in large collections of text documents, referred to as corpora.
It does this by inferring the latent distribution of topics for each document based only on the word tokens, without any supervised labeling.
In LDA, each document $d \in \{1,..,D\}$ is assigned a probability vector $\v{\theta}_d$, where $\theta_{d,k} = \P(\text{topic } k \text{ appears in document } d)$ -- see Table \ref{notation} for a summary of the notation used in this paper.
Then, within document $d$, each word, or token, at position $i$ is assigned a topic indicator $z_{i,d} \in \{1,..,K\}$.
If word $w_{i,d}$ is assigned to topic $k$, it is assumed to be randomly drawn from a probability vector $\v{\phi}_k$ over possible words within that topic.

LDA makes two key exchangeability assumptions.

\begin{enumerate}[(1)]
\item All documents are exchangeable.
\item Within each document, word tokens are exchangeable (bag of words).
\end{enumerate}

These assumptions yield a likelihood for the generative process described in Figure \ref{lda-dag}, which is then combined with conjugate Dirichlet priors on $\v{\phi}_k$ and $\v{\theta}_d$ to form a hierarchical Bayesian model.
The posterior distribution describes, for every word, which topic the word belongs to.

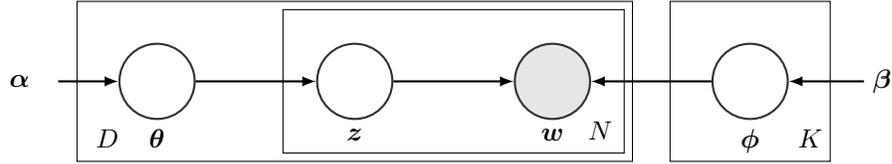
\begin{figure*}[t!]
\begin{center}
\begin{tikzpicture}
\tikzstyle{main}=[circle, minimum size = 10mm, thick, draw =black!80, node distance = 16mm]
\tikzstyle{connect}=[-latex, thick]
\tikzstyle{box}=[rectangle, draw=black!100]
\node[main, draw=none, fill=none] (alpha) {$\v{\alpha}$};
\node[main] (theta) [right=of alpha,xshift=-8mm,label=below:$\v{\theta}$] { };
\node[main] (z) [right=of theta,label=below:$\v{z}$] {};
\node[main, fill = black!10] (w) [right=of z,label=below:$\v{w}$] { };
\node[main] (phi) [right=of w,label=below:$\v{\phi}$] { };
\node[draw=none,fill=none,right=of phi] (beta) {$\v{\beta}$};
\path
(alpha) edge [connect] (theta)
(theta) edge [connect] (z)
(z) edge [connect] (w)
(beta) edge [connect] (phi)
(phi) edge [connect] (w);
\node[rectangle, inner sep=0mm, fit= (z) (w),label=below right:$N$, yshift=1mm, xshift=11.5mm] {};
\node[rectangle, inner sep=0mm, fit= (phi),label=below right:$K$, xshift=0mm] {};
\node[rectangle, inner sep=0mm, xshift=-25mm, fit= (theta) (z) (w), label=below left:$D$] {};
\node[rectangle, inner sep=4.4mm,draw=black!100, fit= (z) (w)] {};
\node[rectangle, inner sep=5.5mm, draw=black!100, fit = (theta) (z) (w) ] {};
\node[rectangle, inner sep=5.5mm, draw=black!100, fit = (phi)] {};
\end{tikzpicture}

\caption{Directed acyclic graph for LDA.}
\label{lda-dag}
\end{center}
\end{figure*}

\begin{table*}[t!]
\begin{center}
\small{
\begin{tabular}{c l c c l}
\hline
Symbol & Description && Symbol & Description
\\
\cline{1-2} \cline{4-5}
$V$ & Vocabulary size && $\m{\Phi} : K \cross V$ & Word-topic probabilities
\\
$D$ & Total number of documents && $\v{\phi}_k : 1 \cross V$ & Word probabilities for topic $k$
\\
$N$ & Total number of words && $\v{\beta} : 1 \cross V$ & Prior concentration vector for $\v{\phi}_k$
\\
$K$ & Total number of topics && $\m{n} : K \cross V$ & Topic-word sufficient statistic
\\
$v(i)$ & Word type for word $i$ && $\m{\Theta} : D \cross K$ & Document-topic probabilities
\\
$d(i)$ & Document for word $i$ && $\v{\theta}_d : 1 \cross K$ & Topic probabilities for document $d$
\\
$w_{i,d}$ & Word $i$ in document $d$ && $\v{\alpha} : 1 \cross K$ & Prior concentration vector for $\v{\theta}_d$
\\
$z_{i,d}$ & Topic indicator for word $i$ in $d$ && $\m{m} : D \cross K$ & Document-topic sufficient statistic
\\
\hline
\end{tabular}
}
\end{center}
\caption{Notation for LDA. Sufficient statistics are conditional on algorithm's current iteration. Bold symbols refer to matrices, bold italic symbols refer to vectors.}
\label{notation}
\end{table*}

To train the model, we need to draw samples from the posterior distribution.
This can be done in a variety of ways, most of which belong to the following broad classes of methods.

\begin{itemize}

\item \emph{Markov Chain Monte Carlo} (MCMC) techniques \cite{griffiths04, magnusson15, li14, yuan15, chen16, newman09, yao09} such as Gibbs sampling can be used, and yield samples from the posterior upon convergence.

\item \emph{Variational Bayesian} (VB) methods \cite{blei03, hoffman13} are also used, which upon convergence yield samples from a variational approximation of the posterior.

\item \emph{Expectation-Maximization} (EM) techniques \cite{li16} are also applicable, which converge to the max a posteriori approximation of the posterior.

\end{itemize}

These are by no means exhaustive: other approaches, such as \emph{spectral} \cite{anandkumar12} and \emph{geometric} \cite{yurochkin16, yurochkin17} techniques can also be applied.
Some methods \cite{hoffman13} focus on rapidly learning the document-topic proportions $\m\Phi$ by avoiding batch calculations with the topic indicators $\v{z}$, whereas others -- including this work -- assume $\v{z}$ itself is one of the quantities of user interest and attempt to calculate the full posterior as efficiently as possible.

We focus here on MCMC techniques -- empirically, their scalability in topic models is comparable with VB and EM, and subject to convergence, they are exact, i.e., they yield the correct posterior.

LDA is increasingly being used with large data sets that are best analyzed in parallel and distributed computational environments.
To be efficient in such settings, a sampler must do the following.

\begin{enumerate}[(1)]

\item Expose sufficient parallelism in a way compatible with the hardware being used.

\item Take advantage of sparsity found in natural language \cite[Ch.\ 2]{zipf68} to control asymptotic computational complexity and memory requirements.

\end{enumerate}

This makes training models a challenge, because MCMC techniques (along with VB and EM) are iterative algorithms that tend to be inherently sequential.
Furthermore, though existing techniques are capable of utilizing sparsity in the sufficient statistics $\m{m}$ and $\m{n}$ for $\m\Phi$ and $\m\Theta$, current approaches can only use sparsity in one sufficient statistic or the other to reduce computational cost, rather than in both simultaneously.
This means that the cost depends on whichever one is \emph{less sparse}.

To make this precise, we call an algorithm \emph{singly sparse} if its iterative computational complexity is at least $O\sbr[1]{\sum_{i=1}^N K_{d(i)}^{(\bigcdot)}}$ or $O\sbr[1]{\sum_{i=1}^N K_{v(i)}^{(\bigcdot)}}$, where $K_{d(i)}^{(\m{m})}$ is the number of nonzero topics in the row of $\m{m}$ corresponding to word $w_{i,d}$, $K_{v(i)}^{(\m{n})}$ is the number of nonzero topics in the column of $\m{n}$ corresponding to word $w_{i,d}$, and $K_{d(i)}^{(\m\Theta)}$, $K_{v(i)}^{(\m\Phi)}$ are defined analogously.
In contrast, we call an algorithm \emph{doubly sparse} if its expected iterative complexity is at most $O\sbr[1]{\sum_{i=1}^N \min\{K_{d(i)}^{(\bigcdot)}, K_{v(i)}^{(\bigcdot)}\}}$ -- such an algorithm can use sparsity in both sufficient statistics simultaneously in such a way that its computational cost depends on whichever one is \emph{more sparse}.

This is particularly the case in the standard MCMC approach -- the collapsed Gibbs sampler of \textcite{griffiths04} modified according to \textcite{yao09} -- where $\m{\Phi}$ and $\m{\Theta}$ are integrated out analytically, and each topic indicator $z_{i,d}$ is sampled one by one conditional on all other topic indicators $\v{z}_{-i,d}$.
In terms of computational cost, the resulting algorithm is singly sparse due to dependence, and exposes very little parallelism, which limits scalability considerably.

We build on the sparse partially collapsed sampler of \textcite{magnusson15}.
This algorithm is exact, parallelizable by documents, and capable of utilizing sparsity in the sufficient statistic $\m{m}$ for $\m\Theta$.
Further, it possesses favorable iterative complexity, and its convergence rate has been empirically shown to be reasonably fast.
Unfortunately, it is singly sparse in the sense that it cannot use sparsity in $\m{n}$, and unlike the standard approach it requires storage of $\m{\Phi}$, a large dense matrix.

Our contribution is a method to bypass this issue by introducing a P\'{o}lya urn approximation.
This allows us to construct a doubly sparse sampler with lower memory requirements, faster runtime, and more efficient iterative complexity.

\section{Previous Work} \label{previous-work}

The standard approach for LDA is the fully collapsed sampler of \textcite{griffiths04}, modified according to \textcite{yao09} to utilize sparsity in the sufficient statistics $\m{n}$ and $\m{m}$.
This sampler has iterative sampling complexity $O\sbr[1]{\sum_{i=1}^N \max\{K_{d(i)}^{(\m{m})}, K_{v(i)}^{(\m{n})}\}}$ -- singly sparse due to the presence of $\max\{..\}$ -- and mixing rate that appears reasonable empirically.
A variety of methods for scalable LDA, including in parallel and distributed environments, have been proposed.

\emph{Partially Collapsed Samplers}.
In this class, $\m{\Theta}$ is integrated out analytically but $\m{\Phi}$ is retained.
Unlike the fully collapsed sampler, the partially collapsed sampler can be parallelized over documents.
This is possible due to the exchangeability assumptions and conjugacy structure present in the LDA model -- a general feature of many Bayesian models \cite{terenin18}.
In \textcite{magnusson15} it is shown that this class can utilize sparsity in a fashion similar to \textcite{li14} to improve sampling speed and yield a singly sparse iterative complexity of $O\sbr[1]{\sum_{i=1}^N K_{d(i)}^{(\m{m})}}$.

Partially collapsed samplers were initially studied in the parallel setting by \textcite{newman09} but were quickly dismissed for having slower convergence rates.
This claim is incorrect: we exhibit an explicit counterexample in Appendix A in which a partially collapsed Gibbs sampler converges arbitrarily faster than a fully collapsed one.
Convergence rate is thus model and algorithm-specific: for LDA, empirical results in \textcite{magnusson15} show that partially and fully collapsed samplers have comparable convergence behavior.

\emph{Asynchronous Distributed LDA}.
\textcite{newman09} propose modifying the collapsed Gibbs sampler presented by \textcite{griffiths04} by simply ignoring the sequential requirement.
Unfortunately, the resulting algorithm is not exact.
In \textcite{ihler12}, the authors bound the 1-step transition error -- their analysis suggests that it is unlikely to accumulate over multiple iterations, and that it is likely to decrease with added parallelism.
However, recent theory on Asynchronous Gibbs Sampling in \textcite{terenin16a} suggests the opposite -- that the approximation error will increase with the added parallelism, because the algorithm will miscalculate the posterior's covariance structure and jump into regions of low probability.
AD-LDA has been studied empirically in \textcite{magnusson15}, where it was shown that AD-LDA can stabilize at a lower log-posterior compared with other samplers, suggesting that asynchronicity can cause it to fail to find a good mode of the posterior distribution.
We are unsure about how to reconcile these differences and recommend further empirical study of AD-LDA's performance, which may depend on the posterior's covariance structure and therefore on the specific corpus being studied.

\emph{Metropolis-within-Gibbs Samplers}. \textcite{li14}, \textcite{yuan15}, and \textcite{chen16} propose methods in which Gibbs steps are replaced with Metropolis-Hastings steps.
These modifications can reduce the iterative sampling complexity per document to $O\sbr[1]{\sum_{i=1}^N K_{d(i)}^{(\m{m})}}$, the number of topics in each document, or even $O(N)$.
However, the use of such steps can affect the MCMC algorithm's mixing rate in ways that are difficult to understand theoretically.
As a result, the combined effect on iterative complexity together with convergence rate is unknown and may depend on the specific corpus.
These ideas are applicable to both fully collapsed and partially collapsed Gibbs samplers, and can be used in combination with the ideas that we present below.

\section{The Algorithm} \label{sec:the-algorithm}

We begin by introducing the basic partially collapsed Gibbs sampler in \textcite{newman09} and \textcite{magnusson15}.
Let $\m{m}^{-i}$ be the sufficient statistic matrix $\m{m}$ for $\m{\Theta}$ with the portion of the statistic corresponding to word $i$ removed.

\begin{algorithm}[Partially Collapsed Gibbs Sampler]
\label{PC-LDA}
Repeat until convergence:
\begin{itemize}
\item Sample $\v{\phi}_k \dist[Dir](\v{n}_k + \v{\beta})$ in parallel for $k=1,..,K$.
\item Sample $z_{i,d} \propto \phi_{k,v(i)}\,\alpha_k + \phi_{k,v(i)}\, m_{d,k}^{-i}$ in parallel for $d=1,..,D$.
\end{itemize}
\end{algorithm}

\subsection{Notation}

Define
\begin{align}
\beta_{\bigcdot} &= \sum_{v=1}^V \beta_v
&
\v{F}_0 &= \frac{\v{\beta}}{\beta_{\bigcdot}}
&
n_{k,\bigcdot} &= \sum_{v=1}^V n_{k,v}
&
\v{\hat{F}}_k &= \frac{\v{n}_k}{n_{k,\bigcdot}}
\,.
\end{align}

The Dirichlet distribution may be parameterized either with a single probability vector $\v{\beta}$ or with two inputs, a concentration parameter $\beta_{\bigcdot}$ and a mean probability vector $\v{F}_0$ whose product is $\v{\beta}$ -- this permits us in what follows to write $\f{Dir}(\v\beta)$ and $\f{Dir}(\beta_{\bigcdot}, \v{F}_0)$ interchangeably.
We introduce the two-input parametrization here because we consider asymptotics where the concentration parameter approaches $\infty$ while the mean probability vector is held fixed.
In Algorithm 1 this permits us to write either $\v{\phi}_k \dist[Dir](\v{n}_k + \v{\beta})$ or
\begin{equation}
\v{\phi}_k \dist[Dir]\sbr{\beta_{\bigcdot} + n_{k,\bigcdot}, \frac{\beta_{\bigcdot}}{\beta_{\bigcdot}+n_{k,\bigcdot}}\v{F}_0 + \frac{n_{k,\bigcdot}}{\beta_{\bigcdot}+n_{k,\bigcdot}}\v{\hat{F}}_k}
\,.
\end{equation}

Finally, as we consider asymptotic convergence of random variables, we must introduce the necessary notions.
We work exclusively with convergence in distribution, which we metrize via the L\'{e}vy-Prokhorov metric -- though this is not strictly necessary, introducing this metric serves to simplify notation.
As the $V$-dimensional probability simplex is a separable space, since it is a subset of $\R^V$, convergence in the L\'{e}vy-Prokhorov metric is equivalent here to convergence in distribution.
Thus we may write
\[
d_{\f{LP}}(\v{x}, \v{x}^*) \goesto 0
\]
to denote pointwise convergence of the cumulative distribution functions of random vectors $\v{x}$ and $\v{x}^*$.

\subsection{Intuition}

We now introduce a small modification to Algorithm \ref{PC-LDA}.
Instead of sampling from the Dirichlet distribution, we sample from a related distribution -- the P\'{o}lya urn distribution, which arises in analysis of the P\'{o}lya urn model.

Throughout this work, we use ideas originating from the infinite-dimensional Dirichlet process literature to reason about finite-dimensional Dirichlet distributions.
These can be viewed as a special case of a Dirichlet process on finite support, and we make use of this connection to develop ideas.

In the P\'{o}lya urn model, there is an urn containing a set of balls of different colors.
We begin with a finite number of balls, with colors distributed according to some initial distribution.
We draw and remove a ball, and then place two balls of the same color back into the urn.
We repeat this process iteratively, letting $C_t$ be a random variable whose distribution is the distribution of colors inside the urn at time $t$.
The sequence $\{C_t : t \in \N\}$ can be shown exchangeable, and is called a \emph{P\'{o}lya sequence}.
A classical result of \textcite{blackwell73} states that it admits a Dirichlet process as its de Finetti measure.

This suggests a way to bypass $\m{\Phi}$ being dense: instead of drawing a random probability vector from the Dirichlet distribution, we can instead draw a set of IID samples from a P\'{o}lya urn and use them to form a probability vector.
Since we are making a finite number of draws from a discrete distribution, some entries will be zero with non-negligible probability, and hence $\v{\phi}_k$ will be sparse.

To proceed, we need to (a) decide how many draws to make and (b) find a parallel sampling scheme.
Recall the representation of the Dirichlet distribution using IID Gamma random variables \cite{kotz00} -- letting $\gamma_j \dist[G](\varpi_j, 1)$, where $\f{G}$ denotes a Gamma distribution with the shape-scale parametrization, we have that if
\begin{equation}
\v\phi_k = \sbr{\frac{\gamma_1}{\sum_{j=1}^V \gamma_j}, .., \frac{\gamma_V}{\sum_{j=1}^V \gamma_j}}
\end{equation}
then $\v\phi_k \dist[Dir](\v\varpi)$.

Notice that $\gamma_j \in \R^+$, $\E(\gamma_j) = \varpi_j$, and $\Var(\gamma_j) = \varpi_j$.
Consider using $\tilde{\gamma}_j \dist[Pois](\varpi_j)$ as a replacement, for which it is also true that $\E(\tilde{\gamma}_j) = \varpi_j$ and $\Var(\tilde{\gamma}_j) = \varpi_j$.
It will be shown below that doing so will precisely yield probability vectors based on sampling from the P\'{o}lya urn.

This procedure has been used by \textcite{draper17} to analyze large data sets in the context of $A/B$ testing.
It is scalable, as it is just a normalized version of the Poisson sampler described in \textcite{chamandy15} for use in parallel and distributed environments.

\subsection{Convergence}

We now prove that the resulting distribution converges to the original Dirichlet with increasing data size.
We begin with the necessary definition.

\begin{definition}[Poisson P\'{o}lya Urn]
For positive $\varpi$ and probability vector $\v{F}$, let $\v\varpi = \varpi\v{F}$.
We say that
\begin{align}
\v{x} &\dist[PPU](\v\varpi)
&
&\text{or equivalently}
&
\v{x} &\dist[PPU](\varpi, \v{F})
\end{align}
if we have that
\begin{equation}
\v{x} = \sbr{\frac{\tilde{\gamma}_1}{\sum_{j=1}^V \tilde{\gamma}_j}, .., \frac{\tilde{\gamma}_V}{\sum_{j=1}^V \tilde{\gamma}_j}}
\end{equation}
for $\tilde{\gamma}_j \dist[Pois](\varpi_j)$, provided that the sum is nonzero.
\end{definition}

As with the Dirichlet, we consider two parametrizations -- one based on a concentration vector, and another based on a concentration parameter together with a mean probability vector.
Note that this definition is exactly the same as the Gamma representation of the Dirichlet distribution \cite{kotz00}, except for Gammas being replaced with Poissons.
We now introduce our main result.

\begin{theorem}[Poisson P\'{o}lya Urn Asymptotic Convergence] \label{polya-urn-conv}
Let
\begin{align}
\v{x} &\dist[PPU](\varpi, \v{F})
&
&\text{and}
&
\v{x}^* &\dist[Dir](\varpi, \v{F})
.
\end{align}
Then for all $\v{F}$ we have $d_{\f{LP}}(\v{x}, \v{x}^*) \goesto 0$ as $\varpi \goesto \infty$.
\end{theorem}

\begin{proof}
Define a Poisson process $\Pi$ over the unit interval with intensity $\varpi$, and let $\pi \dist[Pois]^+(\varpi)$ be the number of arrivals over the entire interval -- here, $\f{Pois}^+$ refers to the Poisson distribution conditioned to be nonzero.
Partition the unit interval into a set of subintervals $B_j$ with lengths equal to each component of $\v{F}$.
Let $\tilde{\gamma}_j$ be the number of arrivals in each subinterval.
Then
\begin{equation}
\frac{\tilde{\gamma}_j}{\sum_{j=1}^V \tilde{\gamma}_j}
\end{equation} is just the proportion of arrivals of $\Pi$ in the subinterval $B_j$.
Thus, since $\Pi$ is a Poisson process, $\v{x}$ admits the hierarchical representation
\begin{align} \label{hierarchical-representation}
\pi &\dist[Pois]^+(\varpi)
&
\upsilon^{(j)} &\iid \v{F}
&
\v{x} = \frac{1}{\pi} \sum_{j=1}^{\pi} \v{I}_{\upsilon^{(j)}}
,
\end{align}
where $\v{I}_{\upsilon^{(j)}}$ is a vector that is one at $\upsilon^{(j)}$ and zero everywhere else.
A random variable admitting this hierarchical representation can be shown to converge to the desired Dirichlet -- proving this is more technical, so we defer the remainder of the proof to Appendix B, noting that the result holds for all $\v{F}$, regardless of whether $\v{F}$ is fixed or variable.
\end{proof}

For LDA we have that under the two-input parametrization
\begin{equation} \label{lda-polya-urn-parameters}
\v\phi_k \dist[PPU]\sbr{\beta_{\bigcdot}+n_{k,\bigcdot}, \frac{\beta_{\bigcdot}}{\beta_{\bigcdot}+n_{k,\bigcdot}} \v{F}_0 + \frac{n_{k,\bigcdot}}{\beta_{\bigcdot}+n_{k,\bigcdot}}\v{\hat{F}}_k}
\,,
\end{equation}
and, noting that $n_{k,\bigcdot}$ is large, we can now define our algorithm.

\begin{algorithm}[P\'{o}lya Urn Partially Collapsed Gibbs Sampler] \label{polya-urn-LDA}
Repeat until convergence:
\begin{itemize}
\item Sample $\v{\phi}_k \dist[PPU](\v{n}_k + \v{\beta})$ in parallel for $k=1,..,K$.
\item Sample $z_{i,d} \propto \phi_{k,v(i)}\,\alpha_k + \phi_{k,v(i)}\, m_{d,k}^{-i}$ in parallel for $d=1,..,D$.
\end{itemize}
\end{algorithm}

By Theorem \ref{polya-urn-conv}, Algorithm \ref{polya-urn-LDA} converges to Algorithm \ref{PC-LDA} as $n_{k,\bigcdot} \goesto \infty$ for all $k$, in the sense that their respective Gibbs steps converge in distribution.
Using our method gives a number of improvements: $\m{\Phi}$ is a sparse matrix and can thus be stored inexpensively, $\m{\Phi}$ can be sampled more efficiently, and the resulting algorithm has better iterative computational complexity for sampling $\v{z}$.
We now describe these improvements in detail.

\subsection{Efficient sampling of $\m{\Phi}$}

The standard partially collapsed sampler requires the rows of $\m{\Phi}$ to be sampled from a Dirichlet distribution.
This is done by drawing Gamma random variables, which can be accomplished with the technique in \textcite{marsaglia00}.

On the other hand, the Poisson P\'{o}lya urn requires us to draw Poisson random variables instead of Gammas -- and, crucially, these Poisson random variables have rate parameters of the form $\beta + l$, $l \in \N_0$.
For $l = 0,..,L$, we can draw these random variables using a precomputed Walker alias table \cite{walker77}.
For $l > L$, we can use the fact that the Poisson distribution is asymptotically Gaussian as its rate parameter increases, and make rounded Gaussian draws.
We choose $L = 100$ and find that the resulting scheme reduces time spent sampling $\m{\Phi}$ significantly -- see Section \ref{sec:results}.

\subsection{Efficient sampling of $\v{z}$}

The sparsity of $\m{\Phi}$ can be used to improve sampling speed for $\v{z}$.
Consider the approach in \textcite{magnusson15}, where each probability is divided into $a$ and $b$, as follows:
\begin{align}
\P(z_{i,d}=k \given \v{z}_{-i,d},\m{\Phi},w_{i,d}) &\propto \phi_{k,v(i)}\sbr{\alpha_k+m_{d,k}^{-i}}
\nonumber
\\
&\propto \underbrace{\phi_{k,v(i)}\,\strut\alpha_k}_{a}+\underbrace{\phi_{k,v(i)}\,\strut m_{d,k}^{-i}}_{b}
.
\end{align}

To sample $z_i$, we draw a random variable $u \dist[U](0,1)$ and compute $u_\sigma = u \del{\sigma_a + \sigma_b}$, where
\begin{align}
\sigma_a &= \sum_{k=1}^K \phi_{k,v(i)} \, \alpha_k
&
&\text{and}
&
\sigma_b &= \sum_{k=1}^K \phi_{k,v(i)} \, m_{d,k}^{-i}
.
\end{align}

If $u_\sigma < \sigma_a$, we draw $z_i$ from a precomputed alias table, otherwise we iterate over $\phi_{k,v(i)} \, m_{d,k}^{-i}$.
Here we can use the fact that both $\m{\Phi}$ and $\m{m}$ are sparse, and iterate over whichever is smaller, to get a doubly sparse algorithm.
This yields improved performance -- see Section \ref{sec:results} -- and also produces improved computational complexity, which we subsequently show.
The sparsity in $\m{\Phi}$ also makes it possible to construct sparse alias tables for $a$, further reducing the memory requirements for the algorithm.

\subsection{Computational Complexity}
\label{sec:comp_complex}

We now derive the iterative computational complexity of P\'{o}lya Urn LDA.
To do so, we need to consider how quickly the number $V$ of unique words grows with the number $N$ of total words.
We assume the following.
\begin{assumption}[Heaps' Law]
The number of unique words in a document follows Heaps' Law $V = \xi N^\psi$ with constants $\xi > 0$ and $0 < \psi < 1$.
\end{assumption}

For most languages, these constants will be in the ranges $5 < \xi < 50$ and $0.4 < \psi < 0.6$ -- see \textcite{araujo97,heaps78} for more details.
Let $K_{d(i)}^{(\m{m})}$ be the number of existing topics in document $d$ associated with word $i$.
Let $K_{v(i)}^{(\m\Phi)}$ be the number of nonzero topics in the row of $\m{\Phi}$ corresponding to word $i$.
Note that $\E\sbr[1]{K_{v(i)}^{(\m\Phi)}} = K_{v(i)}^{(\m{n})}$ at any given iteration.
We initially assume that the number of topics is fixed, and subsequently consider the case where it grows with the size of the data.

\begin{theorem}[Computational Complexity] \label{computational-complexity-theorem}
Assuming a vocabulary size following Heaps' Law, the iterative complexity of P\'{o}lya Urn LDA is
\begin{equation} \label{computational-complexity}
O\sbr{\sum_{i=1}^N \min\cbr{K_{d(i)}^{(\m{m})}, K_{v(i)}^{(\m\Phi)}}}
.
\end{equation}
\end{theorem}
\begin{proof}
The iterative complexity of the sampler is equal to the complexity of sampling all $z_{i,d}$ plus the complexity of sampling $\m{\Phi}$.
Provided the number of topics is fixed, we need not consider the latter.
Thus, it suffices to show that the complexity of all $z_{i,d}$ is given by (\ref{computational-complexity}).
Consider a single $z_{i,d}$ with
\begin{align}
\P(z_{i,d}=k \given \v{z}_{-i,d},\m{\Phi},w_{i,d}) &\propto \phi_{k,v(i)}\sbr{\alpha_k+m_{d,k}^{-i}}
\nonumber
\\
&\propto \underbrace{\phi_{k,v(i)}\,\strut\alpha_k}_{a}+\underbrace{\phi_{k,v(i)}\,\strut m_{d,k}^{-i}}_{b}
.
\end{align}

To sample $z_{i,d}$ we need to calculate the normalizing constant
\begin{equation}
q(z) = \sum_{k=1}^K \phi_{k,v(i)}\sbr{\alpha_k+m_{d,k}^{-i}} = \sigma_a + \sigma_b
.
\end{equation}
The first term is identical for all $z_{i,d}$, so it can be precomputed, and its iterative complexity is constant.
The second term
\begin{equation}
\sigma_b = \sum_{k=1}^K \phi_{k,v(i)}\, m_{d,k}^{-i}
\end{equation}
is the product of two sparse vectors.
Suppose that both are stored in a data structure with $O(1)$ access, such as an array or hash map.
By iterating over the smaller vector, the product can be computed with complexity
\begin{equation}
O\sbr{\min\cbr{K_{d(i)}^{(\m{m})}, K_{v(i)}^{(\m\Phi)}}}
\end{equation}
for each $z_{i,d}$, and the result follows.
\end{proof}

Note the presence of $\min\{..\}$ in the complexity, in contrast with other approaches reviewed in Section \ref{previous-work}.
This demonstrates that P\'{o}lya Urn LDA is a doubly sparse algorithm.

The asymptotic complexity of LDA has also been studied in the setting in which the number $K$ of topics is assumed to grow with the number $N$ of total words \cite{magnusson15}, under the following assumption.

\begin{assumption}
The number of topics $K$ is assumed to follow the mean of a Dirichlet process mixture, and thus the number of topics is $\gamma \ln\sbr{1 + \frac{N}{\gamma}}$, where $\gamma$ is the concentration parameter of the Dirichlet process.
\end{assumption}

See \textcite{teh10} for more details.
Our result carries forward directly into this setting.

\begin{corollary}
Assuming a vocabulary size following Heaps' Law, and a number of topics following the mean of a Dirichlet Process mixture, the iterative complexity of P\'{o}lya Urn LDA is
\begin{equation}
O\sbr{\sum_{i=1}^N \min\cbr{K_{d(i)}^{(\m{m})}, K_{v(i)}^{(\m\Phi)}}}
.
\end{equation}
\end{corollary}

\begin{proof}
As before, the iterative complexity is equal to the complexity of sampling all $z_{i,d}$ plus the complexity of sampling $\m{\Phi}$.
As the procedure for sampling $\m\Phi$ under the Poisson P\`{o}lya urn is identical to that under the Dirichlet, save for Gamma random variables being replaced with Poisson random variables, it follows immediately from \textcite{magnusson15} that its contribution to the total complexity is asymptotically negligible under the given assumptions.
The result thus follows from Theorem \ref{computational-complexity-theorem}.
\end{proof}

\section{Performance Results} \label{sec:results}

To study the empirical performance of P\'{o}lya Urn LDA, we implemented our algorithm in Java using the open-source Mallet\footnote{\raggedright See \href{http://mallet.cs.umass.edu}{http://mallet.cs.umass.edu} and \href{https://github.com/lejon/PartiallyCollapsedLDA}{https://github.com/lejon/PartiallyCollapsedLDA}} \cite{mccallum02} framework.
We ran the algorithm on the Enron, New York Times, and PubMed corpora\footnote{\raggedright Enron and PubMed can be found at: \href{http://archive.ics.uci.edu/ml/datasets/Bag+of+Words}{http://archive.ics.uci.edu/ml/datasets/Bag+of+Words} \\
The New York Times Corpus can be found at: \href{https://catalog.ldc.upenn.edu/ldc2008t19}{https://catalog.ldc.upenn.edu/ldc2008t19}} -- see Table \ref{corpora}.

Computation for each experiment was performed on resources provided by Link\"oping University at the National Supercomputer Centre using an HP Cluster Platform 3000 with SL230s Gen8 compute nodes with two 8-core Intel Xeon E5--2660 processors each.
We used symmetric priors with hyper-parameters $\alpha=0.1$ and $\beta=0.01$, and ran each experiment five times with different random number seeds.
These are typical values and enable comparison with previous results in \textcite{magnusson15}.
Unless specified otherwise, each experiment ran on all 16 cores, and with a rare word limit of 10.

In our first experiment, we compare P\'{o}lya Urn LDA with partially collapsed LDA, on which it is based and which was found to be the fastest parallel sampler in \textcite{magnusson15}.
We ran 1,000 iterations with $K = 10, 100, 1000$ with five different random number seeds.

\begin{table}[t!]
\begin{center}
\begin{tabular}{l r r r}
Corpus & $V$ & $D$ & $N$\\
\hline
Enron & 27 508 & 39 860 & ~ 6 377 365\\
New York Times & 273 405 & 1 825 116 & 494 687 810\\
PubMed & 89 987 & ~ 8 199 999& ~ 768 434 972\\
\hline
\end{tabular}
\vspace*{1ex}
\caption{Corpora used in experiments.}
\label{corpora}
\end{center}
\end{table}

\begin{figure*}
\includegraphics[width=\textwidth]{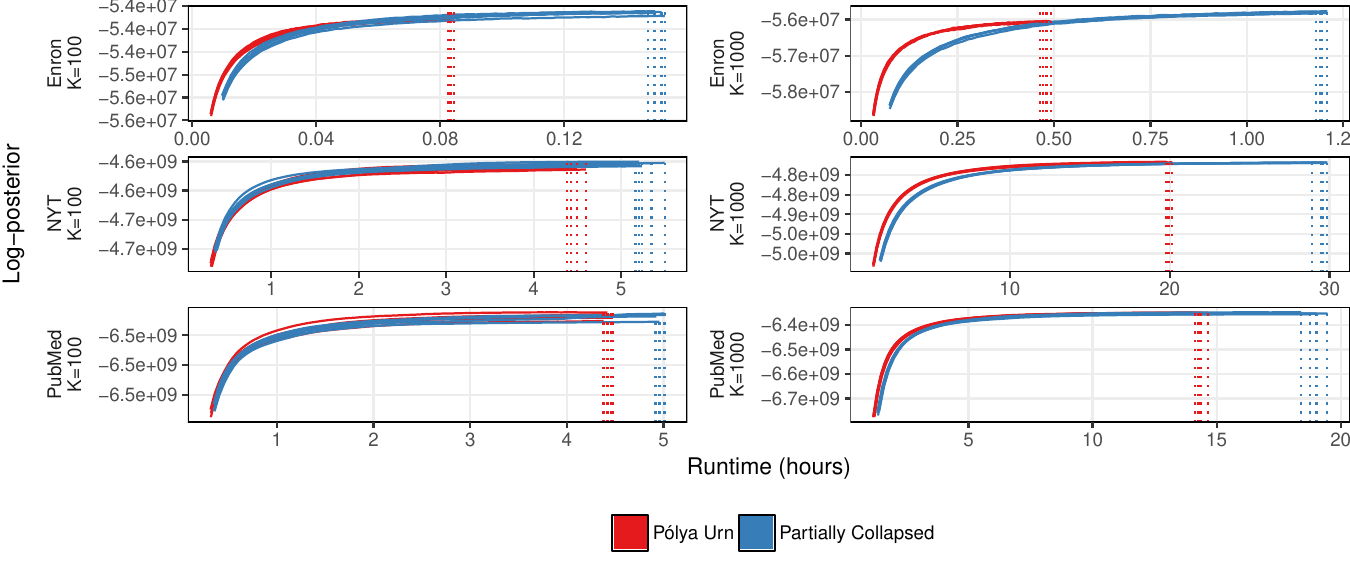}
\caption{Log-posterior trace plots for standard partially collapsed LDA and P\'{o}lya Urn LDA, on a runtime scale (hours). Dashed line indicates completion of 1,000 iterations.}
\label{conv-realtime}
\end{figure*}

\begin{figure*}[t!]
\includegraphics[width=\textwidth]{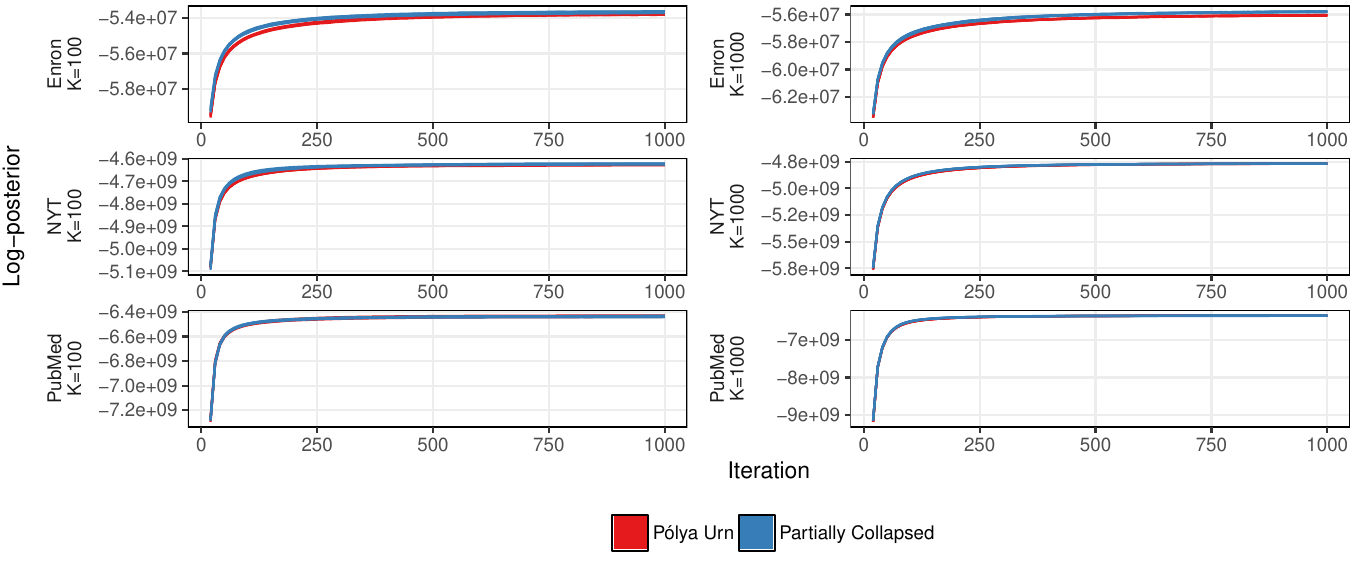}
\caption{Log-posterior trace plots for standard partially collapsed LDA and P\'{o}lya Urn LDA, on a per iteration scale.}
\label{conv-iter}
\end{figure*}

\begin{figure*}[t!]
\bigskip
\includegraphics[width=0.5\textwidth]{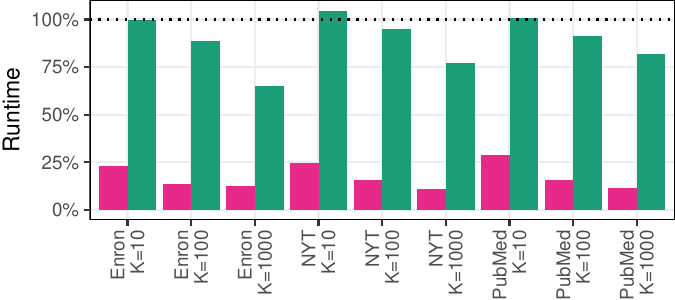}
\includegraphics[width=0.5\textwidth]{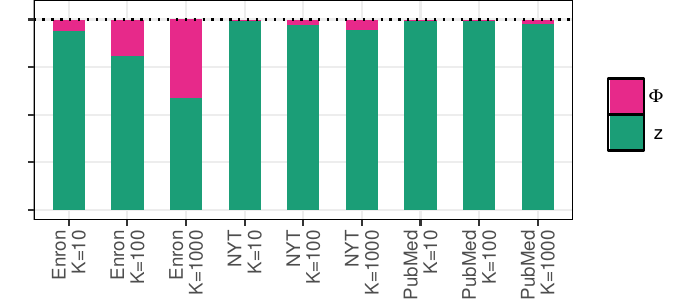}
\caption{Left: runtime for $\m{\Phi}$ and $\v{z}$ for P\'{o}lya Urn LDA, as a percentage of standard partially collapsed LDA -- lower values are faster. Right: percentage of runtime taken by $\m{\Phi}$ and $\v{z}$ for P\'{o}lya Urn LDA.}
\label{phi-z-speed}
\end{figure*}

\begin{figure*}[t!]
\bigskip
\includegraphics[width=\textwidth]{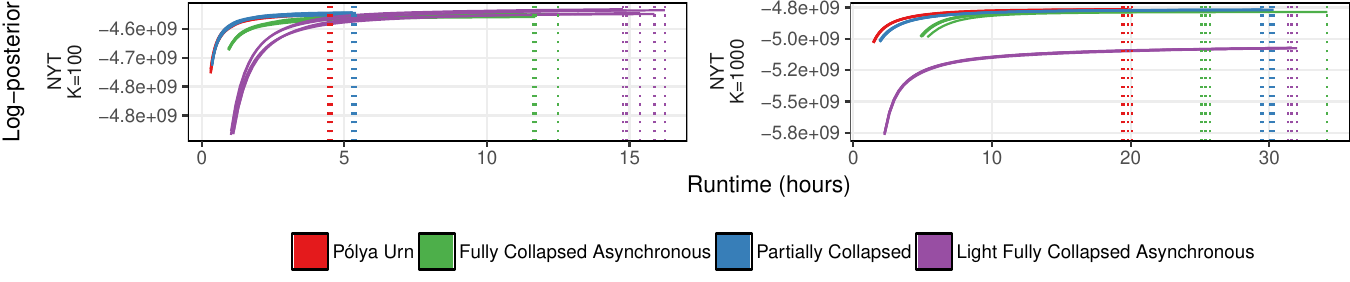}
\caption{Runtime and convergence for P\'{o}lya Urn LDA, Partially Collapsed LDA, Fully Collapsed Sparse LDA, and Fully Collapsed Light LDA, on the NYT corpora.}
\label{algorithm-comparison}
\end{figure*}

\begin{figure*}[t!]
\bigskip
\includegraphics[width=0.5\textwidth]{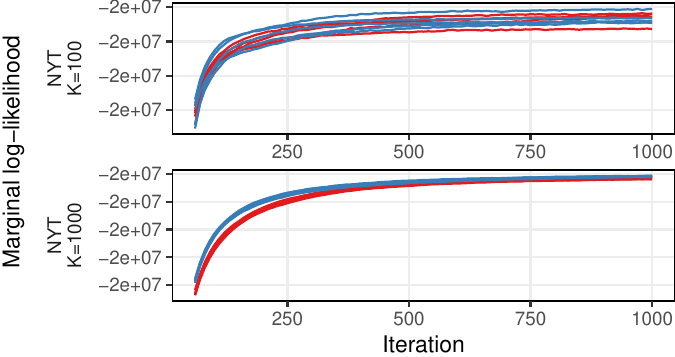}
\includegraphics[width=0.5\textwidth]{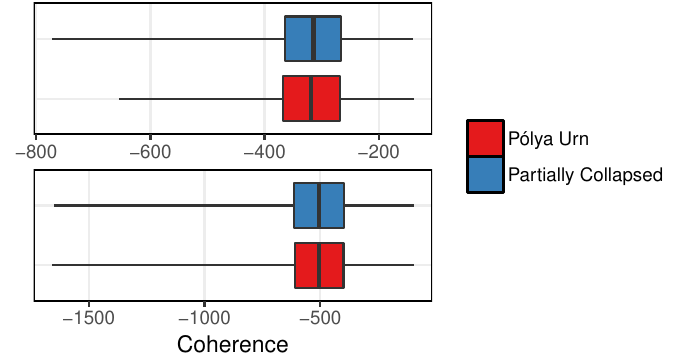}
\caption{Left: test set log-likelihood for P\'{o}lya Urn LDA and Partially Collapsed LDA. Right: test set topic coherence for P\'{o}lya Urn LDA and Partially Collapsed LDA.}
\label{topic-evaluation}
\end{figure*}

\begin{figure*}[t!]
\bigskip
\includegraphics[width=\textwidth]{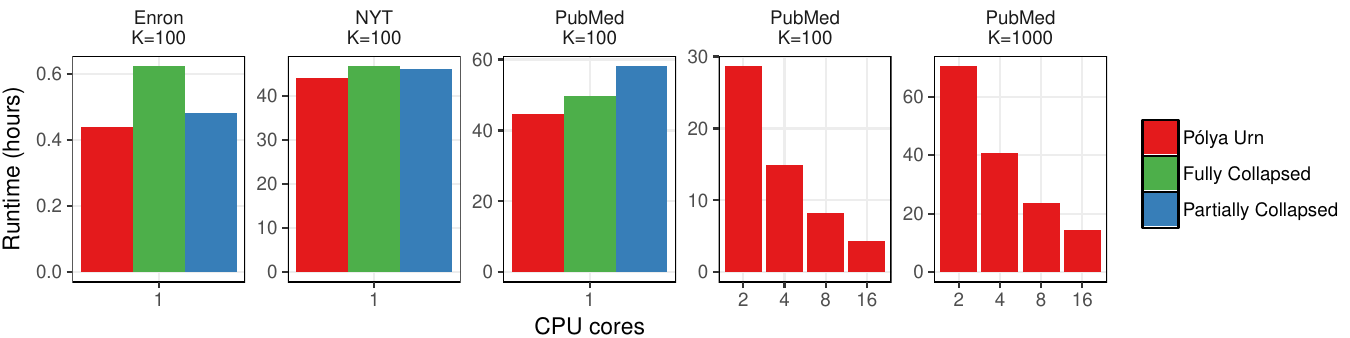}
\caption{Left-most three plots: runtime for P\'{o}lya Urn LDA, Partially Collapsed LDA, and Fully Collapsed Sparse LDA on a single core. Right-most two plots: runtime for P\'{o}lya Urn LDA versus number of available CPU cores.}
\label{runtime-comparison}
\end{figure*}

\begin{figure*}[t!]
\bigskip
\includegraphics[width=\textwidth]{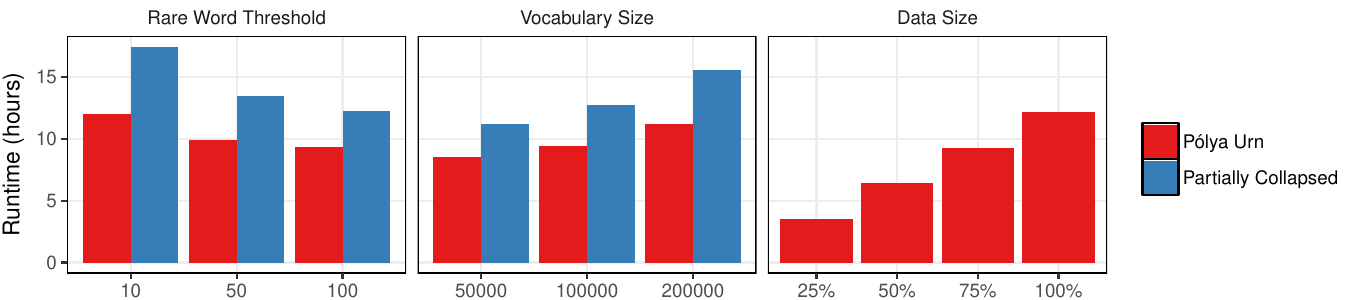}
\caption{Runtime for P\'{o}lya Urn LDA for various rare word thresholds, vocabulary sizes, and data sizes.}
\label{sparsity-comparison}
\end{figure*}

Results can be seen in Figure \ref{conv-realtime}, which gives the marginal unnormalized log-posterior values for $\v{z}$ in real time.
In all cases, P\'{o}lya Urn LDA was faster than standard partially collapsed LDA.
Figure \ref{conv-iter} shows the marginal log-posterior values per iteration.
From this perspective, both algorithms look identical, indicating that P\'{o}lya Urn LDA mixes as quickly as partially collapsed LDA.
We also ran the algorithm for $K=10$: performance of P\'{o}lya Urn LDA and partially collapsed LDA was nearly identical -- this is expected, as $\m{\Phi}$ is too small in this case for the Poisson P\'{o}lya urn to offer any benefit.

Relative runtime for $\m{\Phi}$ and $\v{z}$ separately can be seen in Figure \ref{phi-z-speed} -- as the number of topics increases, the relative speedup for $\v{z}$ increases.
We can also see that sampling $\m{\Phi}$ using the alias tables and the Gaussian approximation in Section \ref{sec:the-algorithm} is uniformly faster by a factor of four to eight.

In our second experiment, we compared P\'{o}lya Urn LDA with partially collapsed LDA \cite{magnusson15}, sparse AD-LDA \cite{yao09}, and Light AD-LDA \cite{yuan15} on the NYT corpora.
Results can be seen in Figure \ref{algorithm-comparison}, which shows that P\'{o}lya Urn LDA is faster for both $K=100$ and $K=1000$.

To ensure that the topics generated by the algorithm are reasonable, we evaluated held-out marginal log-likelihood values for $\v{z}$ for P\'{o}lya Urn LDA and partially collapsed LDA.
For this, we split the NYT corpus into a training and test set at random.
Marginal log-likelihood evaluation on the test set was done using the left-to-right algorithm of \textcite{wallach09}.
We also computed the topic coherence scores of \textcite{mimno11} on NYT -- all of these diagnostics can be seen in Figure \ref{topic-evaluation}.
Both algorithms produced similar test set log-likelihoods, and -- with the exception of one outlier -- similar levels of topic coherence.
This indicates that P\'{o}lya Urn LDA generates topics of similar quality to partially collapsed LDA, and provides empirical evidence that the approximation error induced by the Poisson P\'{o}lya urn is small.

Next, we compared single-core performance of P\'{o}lya Urn LDA with partially collapsed LDA and sparse fully collapsed LDA.
P\'{o}lya Urn LDA was found to outperform fully collapsed sparse LDA and partially collapsed LDA in this setting as well -- results, averaged over all five random number seeds, are given in Figure \ref{runtime-comparison}.
To study parallel scaling, we also looked at performance of P\'{o}lya Urn LDA while varying the number of available processor cores.
We found that the method scales close to linearly with additional overhead for higher levels of parallelism, as expected for a massively parallel algorithm.

Finally, to better understand where the performance improvement comes from, we compared P\'{o}lya Urn LDA with partially collapsed LDA under different data sizes, vocabulary sizes, and rare word thresholds.
For each fixed vocabulary size, words were selected according to their TFIDF values.
Results can be seen in Figure \ref{sparsity-comparison}.
Here, we see that the relative difference between P\'{o}lya Urn LDA and partially collapsed LDA is larger for high vocabulary sizes and lower rare word thresholds.
This is because increasing the dimension of the $\m\Phi$ matrix, whose number of columns is linear in vocabulary size, doesn't slow down $\v{z}$ as much under P\'{o}lya Urn LDA as under partially collapsed LDA.
This indicates that at least some of the speedup is obtained from rare words, for which most entries in the corresponding column of $\m\Phi$ are zero at any given iteration.

Our algorithm's performance depends on the particular data set under study.
Compared to the standard partially collapsed sampler, the P\'{o}lya urn's speedup comes from the additional sparsity in $\m{\Phi}$.
This sparsity will increase with $K$ for a fixed size corpus.
In the PubMed example, our sampling time for $\v{z}$ is reduced, compared to the partially collapsed sampler, by as much as 20\% for $K=1000$ when compared with $K = 100$.
This matches what is expected from our complexity result in Section \ref{sec:comp_complex}, as there will be a substantial number of words where $K_{v(i)}^{(\m\Phi)} < K_{d(i)}^{(\m{m})}$, such as rare words or word types that only belong to a few topics.

\section{Discussion} \label{sec:discussion}

In this paper, we introduce the Poisson P\'{o}lya urn distribution as an asymptotically exact approximation to the Dirichlet distribution and use this to define P\'{o}lya Urn LDA for fast parallel sampling in LDA and similar models.
Our algorithm is \emph{doubly sparse}, in that it can take advantage of sparsity in the sufficient statistics $\m{m}$ and $\m{n}$ for $\m\Phi$ and $\m\Theta$ simultaneously.
It inherits the parallelizability of partially collapsed LDA, which adds the following improvements.
\begin{enumerate}[(1)]
\item The word-topic probability matrix $\m{\Phi}$ is sparse and can be stored efficiently.
\item Sampling $\m{\Phi}$ requires Poisson draws, which can be sampled efficiently -- see Section \ref{sec:the-algorithm}.
\item The topic indicators $\v{z}$ can be sampled efficiently by utilizing the added sparsity in $\m{\Phi}$ -- this improves our iterative complexity to $O\sbr{\sum_{i=1}^N \min\cbr{K_{d(i)}^{(\m{m})}, K_{v(i)}^{(\m\Phi)}}}$.
\end{enumerate}

To our knowledge, this is the lowest computational complexity in any non-Metropolis-Hastings-based LDA sampler.
It enables our algorithm to better take advantage of the power-law structure present in natural language.
Combined with faster Poisson draws, we obtain significant speedup under standard corpora used for benchmarking -- see Section \ref{sec:results}.

The Poisson P\'{o}lya urn is an asymptotic approximation whose error we prove vanishes with increasing data set size.
Our proof is based crucially on the Central Limit Theorem, which suggests that the convergence rate of our approximation is likely to be at least $O(\sqrt{n_{k,\bigcdot}})$, though we have not rigorously demonstrated this.

In practice, we find the approximation error to be negligible, as from a convergence point of view, P\'{o}lya Urn LDA and partially collapsed LDA behave near-identically.
This mirrors the behavior of scalable approximate MCMC methods recently proposed in other areas -- see, e.g., \textcite{johndrow17}.
Indeed, if we assume that the standard partially collapsed Gibbs sampler is geometrically ergodic, together with several other regularity conditions, then the theory in \textcite{negrea17}, as well as \textcite{johndrow15}, provides bounds on the resulting Markov chain's stationary distribution in terms of the correct posterior.
On the other hand, this is also the method's main drawback: the approximation's convergence rate is unlikely to be dimension-free and thus may depend on $K$, so we advise caution when considering it for large $K$ with small data sets.

Our analysis -- and, indeed, virtually every other complexity analysis in the LDA literature of which we are aware -- is limited in that it describes the \emph{iterative complexity} of the algorithm, not the \emph{effective complexity}: it does not take into account the rate at which the Markov chain converges to the posterior, which will differ by algorithm.
To calculate the effective complexity, we would need to prove that the chain is geometrically ergodic and compute the ergodicity coefficient, or find the spectral gap of the corresponding Markov operator.
For LDA, this is currently an open problem.
Similar issues occur in comparing MCMC approaches with Variational Bayes, Expectation-Maximization, and other methods, for which posterior approximation error needs to be considered yet is often challenging to assess.

Our approach can be combined with Metropolis-Hastings proposals such as those described in \textcite{li14} and \textcite{yuan15}.
We did not consider these, as they significantly alter the convergence rate -- this would make the speedup from the Poisson P\'{o}lya urn more difficult to understand.
To compare these approaches with the standard partially collapsed sampler on a runtime basis see Figure \ref{topic-evaluation} and \textcite{magnusson15}.

P\'{o}lya Urn LDA is well-suited to a variety of parallel environments.
Here we focus on the multicore setting and find that P\'{o}lya Urn LDA outperforms partially collapsed LDA, which in turn was found to outperform other approaches in \textcite{magnusson15}.
Other parallel environments may yield different results: in particular, AD-LDA may perform better in the compute cluster setting, where asynchronous delays may serve to counteract the algorithm's bias -- see \textcite{terenin16a}.
P\'{o}lya Urn LDA could also be implemented in such an environment -- in fact, it can potentially be implemented using the \emph{Exact Asynchronous Gibbs} scheme in \textcite{terenin16a}, so that asynchronous delays introduce no approximation error into the posterior.
These considerations are left to future work.

Developing massively parallel schemes for topic models that are both practical and theoretically understandable is an area of active current research.
Our use of the Poisson P\'{o}lya urn approximation is generic: though we focus on LDA in this paper, the technique is equally applicable to any other topic model or language model making use of the Dirichlet distribution.
This includes generative document cluster models \cite{zhang05}, hidden Markov models used in parts-of-speech tagging \cite{goldwater07}, mixed membership stochastic blockmodels \cite{airoldi08}, and others.
We hope that our work contributes to greater understanding in these areas.

\section*{Acknowledgments}

We are thankful to Graham Neubig, Tamara Broderick, Eric P. Xing, Qirong Ho, Wei Dai, Willie Neiswanger, Murat Demirbas, Shawfeng Dong, Thanasis Kottas, Kunal Sarkhel, and Mattias Villani for their thoughts, and to several referees whose comments substantially improved the paper.
Membership on this list does not imply agreement with the ideas expressed here, nor responsibility for any errors that may be present.
M{\aa}ns Magnusson was partially financially supported by the Swedish Foundation for Strategic Research (Smart Systems: RIT 15--0097).

\section*{Appendix A: Efficiency of Collapsed and Uncollapsed Gibbs Samplers}

Here we demonstrate that a collapsed Gibbs sampler can be arbitrarily less efficient than an uncollapsed Gibbs sampler.
Our example is a scaled $\f{T}$ distribution on $\R^2$ with 5 degrees of freedom, with both components having mean 0 and scale matrix $\m{\Sigma}$ such that $\Sigma_{11} = \Sigma_{22} = 1$ and $\Sigma_{21} = \Sigma_{12} = \rho$.
We define our variables $\Phi$ and $\v{z}$ in a fashion analogous to those in LDA.

\begin{algorithm}[Collapsed Gibbs Sampler]\label{t-gibbs}
Repeat until convergence:
\begin{itemize}
\item Sample $z_1 \given z_2 \dist[T][\rho z_2, (0.8 + 0.2z_2^2)(1 - \rho^2), 5]$
\item Sample $z_2 \given z_1 \dist[T][\rho z_1, (0.8 + 0.2z_1^2)(1 - \rho^2), 5]$
\end{itemize}
\end{algorithm}

\begin{algorithm}[Uncollapsed Gibbs Sampler]\label{t-ig-gibbs}
Repeat until convergence:
\begin{itemize}
\item Sample $\Phi \dist[IG] [3, 0.5 \v{z}^T \m{\Sigma}^{-1} \v{z} + 2]$
\item Sample $\v{z} \dist[N]_2 [\v{0}, \Phi \m{\Sigma}]$
\end{itemize}
\end{algorithm}

By taking $\rho$ sufficiently large, Algorithm \ref{t-gibbs} can be made to converge at an arbitrarily slow rate.
However, Algorithm \ref{t-ig-gibbs} will converge at a rate independent of $\rho$.
This can be seen intuitively by noting that for $\rho = 1$, the Markov Chain in Algorithm 1 will be reducible and will thus fail to converge at all, whereas the chain in Algorithm 2 will suffer no such difficulty.
This is illustrated in Figure \ref{t-comparison}, for $\rho = \{0.9,0.99,0.999\}$.
This example demonstrates that a collapsed Gibbs sampler may be arbitrarily slower than an uncollapsed or partially collapsed Gibbs sampler.

This serves as an explicit counterexample to the argument made in \textcite{newman09} on the basis of Theorem 1 in \textcite{liu95}.
The error in this argument is as follows: \textcite{liu95} assume reversibility and derive regularity conditions under which the correlation of two successive samples $(\v{z}^{(i)}, \v{z}^{(i+1)})$ must strictly increase when some other parameter $\Phi$ is inserted between them.
One crucial assumption is that $\v{z}$ is sampled as a block whether $\Phi$ is present or not.
This assumption does not apply to the example in this appendix, nor to LDA, where $\v{z}$ is sampled as a block in one algorithm and component-wise in the other.

More generally, analogous issues have been shown to lead to provably slow mixing in certain large-scale Gibbs samplers -- see \textcite{johndrow16} for details.
It may be possible to apply similar arguments to LDA, or to extend the argument in \textcite{liu95} to the case where the dimensionality of $\v{z}$ grows -- this may show in certain samplers that mixing slows down with data size.
This is a concern for any algorithm, as it would result in less favorable effective complexity -- see Section \ref{sec:discussion}.

\begin{figure}[t!]
\begin{center}
\includegraphics[width=0.5\textwidth]{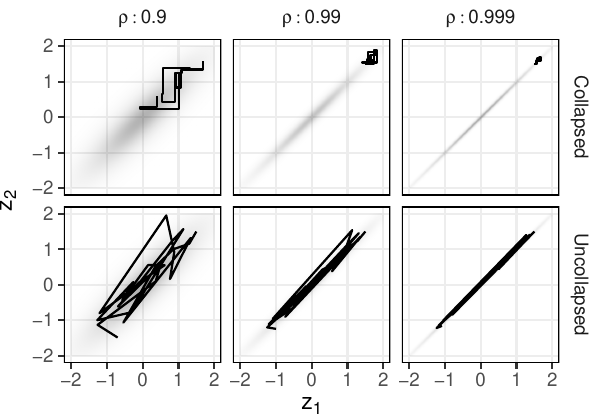}
\end{center}
\caption{Trace plots for the collapsed and uncollapsed Gibbs samplers for a $\f{T}$ distribution on $\R^2$ with $\rho \in \{0.9, 0.99, 0.999\}$, together with target distributions in grayscale. In 25 iterations, the uncollapsed Gibbs sampler has traversed the entire distribution multiple times, whereas the collapsed Gibbs sampler has not done so even once, covering increasingly less distance for larger $\rho$.}
\label{t-comparison}
\end{figure}

\section*{Appendix B: Poisson-P\'{o}lya Urn Asymptotic Convergence}

In this section we prove the remainder of Theorem \ref{polya-urn-conv}.
We first introduce some notation used throughout the rest of Appendix B.
Let $\v{x}, \v{x}^* \in \R^p$ be random vectors, let $\varpi \in \R^+$, and let $\v{F} = (F_1, .., F_p)$ be a probability vector.
We allow $\v{x}$, $\v{x}^*$, and $\v{F}$ to be implicitly indexed by $\varpi$.
The distributions of $\v{x}$ and $\v{x}^*$ will vary by context.
As before, we work exclusively with convergence in distribution, which we again metrize via the L\'{e}vy-Prokhorov metric to simplify notation.
Thus $d_{\f{LP}}(\v{x}, \v{x}^*) \goesto 0$ can be taken to mean that $\v{x}$ and $\v{x}^*$ converge in distribution, i.e., their cumulative distribution functions converge pointwise.
With this notation in mind, we now introduce several lemmas.

\begin{lemma}[Gaussian Asymptotic Distribution of a Dirichlet Random Vector] \label{dirichlet-gaussian}
Let $\v{x} \dist[Dir](\varpi, \v{F})$.
Then there exists a sequence of multivariate Gaussian distributions $\v{x}^*$ indexed by $\varpi$ such that $d_{\f{LP}}(\v{x}, \v{x}^*) \goesto 0$ as $\varpi \goesto \infty$.
\end{lemma}
\begin{proof}
Without loss of generality suppose $\varpi \in \N$.
A Dirichlet random vector admits the representation
\begin{equation}
\v{x} = \sbr{\frac{\gamma_1}{\sum_{j=1}^p \gamma_j}, .., \frac{\gamma_p}{\sum_{j=1}^p \gamma_j}} = \frac{\v\gamma}{\sum_{j=1}^p \gamma_j}
\end{equation}
with $\gamma_i \dist[G](\varpi F_i, 1)$.
Since a sum of gamma random variables is gamma, we may write
\[
\v\gamma = \sum_{i=1}^{\varpi} \v{\hat\gamma_i}
\]
where for each $j$ we have $\hat\gamma_{ij} \iid[G](F_j, 1)$.
Note that $\E(\v{\hat\gamma}_i) = \v{F}$.
It is now clear that taking $\varpi \in \N$ entails no loss in generality, because we may replace $F_j$ with $\varpi F_j / \lceil\varpi\rceil$ where $\lceil\cdot\rceil$ is the ceiling function, and take the summation from $1$ to $\lceil\varpi\rceil$.
Since $\v{\hat\gamma}_i$ have finite moments, we have by the Central Limit Theorem that
\begin{equation}
\sqrt{\varpi}\sbr{\frac{1}{\varpi} \v\gamma - \v{F}} \overset{\f{d}}{\goesto} \f{N}(\v{0},\m\Sigma)
\end{equation}
for some covariance matrix $\m\Sigma$ as $\varpi \goesto \infty$ where $\overset{\f{d}}{\goesto}$ denotes convergence in distribution.
Next, we use the delta method to show this carries over to $\v{x}$.
Define the function $g : \R^p \goesto \R^p$ by $\v{x} = g(\v\gamma)$.
Since $g$ is continuously differentiable for all strictly positive input, it is continuously differentiable at $\v{F}$, as $\v{F}$ is a probability vector.
Then by Theorem 3.1 of \textcite{vandervaart00}, we have
\begin{equation}
\sqrt{\varpi}\sbr[2]{g(\v\gamma / \varpi) - g(\v{F})} \overset{\f{d}}{\goesto} \f{N}(\v{0},\m{\tilde\Sigma})
\end{equation}
for a covariance matrix $\m{\tilde\Sigma}$ as $\varpi \goesto \infty$.
But notice that $g(\v\gamma / \varpi) = g(\v\gamma) = \v{x}$.
Since the L\'{e}vy-Prokhorov metric metrizes convergence in distribution, the result follows.
\end{proof}

\begin{lemma}[Asymptotic Distribution and Moments of the Poisson-P\'{o}lya Urn] \label{polya-urn-multinomial}
Let $\v{x} \dist[PPU](\varpi, \v{F})$ be defined via the hierarchical representation (\ref{hierarchical-representation}).
Let $\v{x}^* \dist[MN](\varpi, \v{F})$, where $\f{MN}$ denotes the multinomial distribution.
Then, as $\varpi \goesto \infty$, we have that $d_{\f{LP}}(\varpi \v{x}, \v{x}^*) \goesto 0$ and for all $i,j$ that
\begin{align}
|\E(x_i) - F_i| &\goesto 0
&
\abs{\Var(x_i) - \frac{F_i (1 - F_i)}{\varpi}} &\goesto 0
\end{align}
\begin{equation}
\abs{\Cov(x_i, x_j) - \frac{-F_i F_j}{\varpi}} \goesto 0 \text{ for } i \neq j
\,.
\end{equation}
\end{lemma}
\begin{proof}
To show convergence in distribution, and convergence of the first two moments, it suffices to show that the probability generating functions of both random variables converge pointwise and are twice differentiable in a neighborhood around $\v{1}$.
This follows immediately from the argument in \textcite{oza05} by replacing their expression $\frac{1}{N}(x_1 + .. + x_N)$ with $\sum_{i=1}^p F_i x_i$, and noting that the probability generating functions of $\f{Pois}$ and $\f{Pois}^+$ random variables converge pointwise as their means go to $\infty$.
\end{proof}

\begin{lemma}[Gaussian Asymptotic Distribution of a Multinomial Random Vector] \label{multinomial-gaussian}
Let $\v{x} \dist[MN](\varpi, \v{F})$.
Then there exists a sequence of multivariate Gaussian distributions $\v{x}^*$ indexed by $\varpi$ such that $d_{\f{LP}}(\v{x}, \v{x}^*) \goesto 0$ as $\varpi \goesto \infty$.
\end{lemma}

\begin{proof}
Since a multinomial can be expressed as the sum of IID discrete distributions, the result follows from the Central Limit Theorem.
\end{proof}

\noindent
We can now complete the convergence proof in Section \ref{sec:the-algorithm}.

\begin{theorem}[Poisson P\'{o}lya Urn Asymptotic Convergence] \label{polya-urn-conv-appendix}
Let $\v{x} \dist[PPU](\varpi, \v{F})$.
Let $\v{x}^* \dist[Dir](\varpi, \v{F})$.
Then $d_{\f{LP}}(\v{x}, \v{x}^*) \goesto 0$ as $\varpi \goesto \infty$ for all $\v{F}$.
\end{theorem}

\begin{proof}
Using Lemma \ref{polya-urn-multinomial}, Lemma \ref{multinomial-gaussian}, and Lemma \ref{dirichlet-gaussian}, it follows that both random vectors converge in distribution to multivariate Gaussians.
Since Gaussians are uniquely characterized by their first two moments, it suffices to show for all $i,j$ that
\begin{align}
|\E(x_i) - F_i| &\goesto 0
&
\abs{\Var(x_i) - \frac{F_i (1 - F_i)}{\varpi + 1}} &\goesto 0
\end{align}
\begin{equation}
\abs{\Cov(x_i, x_j) - \frac{-F_i F_j}{\varpi + 1}} \goesto 0 \text{ for } i \neq j
\end{equation}
as $\varpi \goesto \infty$, since these are the mean and covariance of the Dirichlet.
Since $\frac{\varpi}{\varpi+1} \goesto 1$, the differences between these moments and those in Lemma \ref{polya-urn-multinomial} approach zero, and the result follows.
\end{proof}

\printbibliography

\begin{IEEEbiographynophoto}
Alexander Terenin is a Ph.D student in the Statistics section of the Department of Mathematics at Imperial College London. He was previously at Petuum, Inc. and eBay, Inc. His research uses Bayesian theoretical tools to understand the models used in today's big data settings, while also taking the systems perspective needed for high performance on modern hardware, with a particular focus on Markov Chain Monte Carlo methods in parallel and distributed environments.
\end{IEEEbiographynophoto}

\begin{IEEEbiographynophoto}
M{\aa}ns Magnusson received his MSc degree in Statistics (2011) from Stockholm University. M{\aa}ns has previously worked as a statistician at the Swedish Agency for Education, The Swedish Agency for Crime Prevention and The Swedish Agency for Communicable Disease Control. In 2012 he started his PhD studies at Link\"{o}ping University. His research interests include Bayesian Methods for analyzing textual data and his current focus area is different aspects of topic modeling for large corpora in computational social science and digital humanities.
\end{IEEEbiographynophoto}

\begin{IEEEbiographynophoto}
Leif Jonsson received his MSc degree in Computer Science (1998) from Uppsala University, in the same year he started working at Ericsson AB's research division. In 2010 he started his PhD studies at Link\"{o}ping University. His research interests include applying machine learning techniques to large-scale software development processes to automate traditionally hard to automate tasks. His current focus area is automatic fault localisation and bug triaging.
\end{IEEEbiographynophoto}

\begin{IEEEbiographynophoto}
David Draper is a Professor of Statistics in the Department of Applied Mathematics and Statistics at the University of California, Santa Cruz. His data science consultative relationships have included work with eBay, Amazon, and SoFi. His current research interests include data science, model specification, and fast accurate computation at Big-Data scale, all from a Bayesian perspective.
\end{IEEEbiographynophoto}

\end{document}